\documentclass[pdflatex,sn-mathphys-num]{sn-jnl}

\usepackage{graphicx}%
\usepackage{multirow}%
\usepackage{amsmath,amssymb,amsfonts}%
\usepackage{amsthm}%
\usepackage{mathrsfs}%
\usepackage[title]{appendix}%
\usepackage{xcolor}%
\usepackage{textcomp}%
\usepackage{multirow}
\usepackage{manyfoot}%
\usepackage{booktabs}%
\usepackage{algorithm}%
\usepackage{algorithmicx}%
\usepackage{algpseudocode}%
\usepackage{listings}%

\theoremstyle{thmstyleone}%
\newtheorem{theorem}{Theorem}
\newtheorem{proposition}[theorem]{Proposition}%

\theoremstyle{thmstyletwo}%

\theoremstyle{thmstylethree}%

\usepackage{graphicx}
\usepackage{bm}
\usepackage{amssymb}
\usepackage{amsmath}
\usepackage{booktabs}
\usepackage{calc}
\usepackage[color=blue!30]{todonotes}
\usepackage{algorithm}
\usepackage{algpseudocode}
\usepackage[capitalise]{cleveref}
\usepackage{xspace}
\usepackage{url}
\usepackage{pgfplots}

\algrenewcommand\algorithmicrequire{\textbf{Input:}\phantom{sa}}
\algrenewcommand\algorithmicensure{\textbf{Output:}}
\algrenewtext{EndFor}{\textbf{end}}

\newcommand{\TS}{time-series\xspace}
\newcommand{\ourmethod}{\textbf{AALF}\xspace}
\newcommand{\fint}{f}
\newcommand{\fcomp}{g}
\newcommand{\abs}[1]{\lVert #1 \rVert}

\newcommand{\bxt}{\bm{x}_t}

\newcommand{\X}[1]{x^{(#1)}}
\newcommand{\bX}[1]{\bm{x}^{(#1)}}

\DeclareMathOperator*{\argmin}{argmin}

\definecolor{color1}{HTML}{377eb8}
\definecolor{color2}{HTML}{E41A1C}
\definecolor{color3}{HTML}{4daf4a}
\definecolor{color4}{HTML}{ff7f00}

\begin{document}

\title[AALF]{AALF: Almost Always Linear Forecasting}


\author*[1,2]{\fnm{Matthias} \sur{Jakobs}}\email{matthias.jakobs@tu-dortmund.de}

\author[1,2]{\fnm{Thomas} \sur{Liebig}}\email{thomas.liebig@tu-dortmund.de}


\affil[1]{\orgname{Lamarr Institute for Machine Learning and Artificial Intelligence}, \orgaddress{\city{Dortmund}, \country{Germany}}}
\affil[2]{\orgname{TU Dortmund University}, \orgaddress{\city{Dortmund}, \country{Germany}}}



\abstract{Recent works for \TS forecasting more and more leverage the high predictive power of Deep Learning models.
With this increase in model complexity, however, comes a lack in understanding of the underlying model decision process, which is problematic for high-stakes application scenarios.
At the same time, simple, interpretable forecasting methods such as ARIMA still perform very well, sometimes on-par, with Deep Learning approaches.
We argue that simple models are good enough most of the time, and that forecasting performance could be improved by choosing a Deep Learning method only for few, important predictions, increasing the overall interpretability of the forecasting process.
In this context, we propose a novel online model selection framework which learns to identify these predictions. An extensive empirical study on various real-world datasets shows that our selection methodology performs comparable to state-of-the-art online model selections methods in most cases while being significantly more interpretable.
We find that almost always choosing a simple autoregressive linear model for forecasting results in competitive performance, suggesting that the need for opaque black-box models in \TS forecasting might be smaller than recent works would suggest.}

\keywords{Time Series Forecasting, Model Selection, Interpretability}



\maketitle

\section{Introduction}
Time-series forecasting has been recognized as a vital component in informed decision-making in many different real world application settings, such as predictive maintenance, energy management, financial investments and smart city transport planing \cite{godahewaMonashTimeSeries2021,hyndmanStateSpaceFramework2002,saadallahBRIGHTDriftAwareDemand2020}.
Due to the complex and dynamic characteristics of \TS, forecasting is considered one of the most challenging parts in \TS data analysis.
Numerous Machine Learning approaches have been proposed to tackle the forecasting problem.
First, many works tackle forecasting in an online or streaming fashion \cite{hyndmanLargescaleUnusualTime2015}, utilizing previously observed values (and predictions) as features.
Others create feature spaces by considering the $L$ last known values, referred to as the lagged values, as fixed-length embeddings used for predicting future values \cite{jakobsExplainableAdaptiveTreebased2023,saadallahExplainableOnlineDeep2021,saadallahExplainableOnlineEnsemble2022,cerqueiraArbitratedEnsembleTime2017,saadallahOnlineDeepHybrid2023}.
However, as a direct consequence of the No Free Lunch theorem \cite{wolpertNoFreeLunch1997}, we know that no model can be universally applied to all application settings all the time.
Especially for \TS data whose characteristics might change over time we know that models exhibit varying performance during the forecasting task due to concept drifts \cite{jakobsExplainableAdaptiveTreebased2023,saadallahExplainableOnlineDeep2021,saadallahExplainableOnlineEnsemble2022,saadallahOnlineDeepHybrid2023}.
One way to tackle this problem is to utilize online model selection or ensembling methods to try to find the best models for each time step during forecasting.

In recent years, researchers have become more conscious of the downsides that come with applying Deep Learning methods in practice, namely poor understanding of why these models perform in the way they do.
Especially in safety-critical application settings, it is important to understand the model behavior to be able to audit and verify that the predictions are not made due to spurious correlations and in a fair matter \cite{guidottiSurveyMethodsExplaining2018}.
In the literature, these aspects are often discussed under the terms \emph{Interpretability} and \emph{Explainability}.
We use these two terms based on the definitions of Rudin \cite{rudinStopExplainingBlack2019}: \emph{Interpretability} refers to the model itself being transparent and small enough so that researchers and domain experts can understand their inner workings completely, whereas \emph{Explainability} refers to methods auditing black-box models by generating explanations of why a model predicted specific datapoints the way it did after the fact.
We argue, as others \cite{rudinStopExplainingBlack2019}, that interpretable methods are superior in comparison to explaining uninterpretable methods post-hoc, since researchers are able to understand and verify them more easily and with greater confidence.

In practice, there is an often observed trade-off between interpretability and predictive performance: The more interpretable a model is the smaller its predictive performance tends to be.
However, in \TS forecasting, this conventional wisdom is often found not to be the case.
For example, in the M4 forecasting competition \cite{makridakisM4CompetitionResults2018}, the organizers found that traditional, simple and interpretable methods such as ARIMA and Exponential Smoothing performed surprisingly well, and that by combining interpretable methods with high performing Deep Learning methods, an overall increase in predictive performance can be observed \cite{smylHybridMethodExponential2020}.

This leads us to the hypothesis that we can increase interpretability while maintaining high predictive performance in \TS forecasting by employing model selection to choose an interpretable model whenever possible, and only resort back to Deep Learning methods if their predictions are superior enough to warrant the decrease in interpretability.
Thus, in this work, we propose the following:

\begin{itemize}
    \item We propose a novel model selection strategy called AALF: Almost Always Linear Forecasting (see \cref{fig:overview}). \ourmethod defines (on historic data) the optimal selection given a interpretability constraint and aims to learn a classifier from these labels. This classifier then selects between an interpretable model (in our case, a simple autoregressive model) and a Deep Learning black-box model for each time step of a forecasting process, preferring the interpretable model whenever possible.
    \item We provide a comparative empirical study with other state-of-the-art model selection algorithms on over $3500$ \TS and discuss the implications in terms of predictive performance and interpretability. We find that often very little help by the black-box model is needed to boost the overall forecasting performance of the autoregressive model significantly.
    \item We show that, by varying the desired amount of interpretable model selection, we can outperform or perform competitively to other state-of-the-art online model selection methods while choosing the interpretable model significantly more often, improving the overall interpretability of the forecasting process.
    \item Our proposed method is a generic framework which is easily adaptable and can be extended to be used in different application scenarios. While we showcase the validity of our method with a autoregressive model and common Deep Learning models we want to stress that we do not make assumptions about the specific model families.
\end{itemize}

\begin{figure*}
    \centering
%
%
%

\usetikzlibrary{patterns,decorations.pathreplacing}
\begin{tikzpicture}

\def\cellwidth{0.5cm}
\def\cellheight{0.5cm}

\foreach \row in {1, 2} {
	\node[left] at (\cellwidth, -\row * \cellheight - 0.5*\cellheight) {$\bm{x}_{\row}$};
	\foreach \col in {1, 2, 3, 4, 5} {
		\draw[thick] (\col*\cellwidth, -\row*\cellheight) rectangle
		(\col*\cellwidth + \cellwidth, -\row*\cellheight - \cellheight);
	}
}

\node[left] at (0, -3*\cellheight - \cellheight/2) {};
\foreach \col in {1, 2, 3, 4, 5} {
	\node at (\col*\cellwidth + 0.5*\cellwidth, -3*\cellheight - 0.5*\cellheight) {$\dots$};
}

\node[left] at (\cellwidth, -4*\cellheight - \cellheight/2) {$\bm{x}_T$};
\foreach \col in {1, 2, 3, 4, 5} {
	\draw[thick] (\col*\cellwidth, -4*\cellheight) rectangle
	(\col*\cellwidth + \cellwidth, -4*\cellheight - \cellheight);
}

\def\yoffset{1.5cm}
\foreach \col in {1, 2} {
	\draw[thick] (\col*\cellwidth, -4*\cellheight - \yoffset) rectangle
	(\col*\cellwidth + \cellwidth, -4*\cellheight - \cellheight - \yoffset);
	\node[above] at ((\col*\cellwidth + 0.5*\cellwidth, -4*\cellheight  -\yoffset) {$y_{\col}$};
}
\node[right] at (3.35*\cellwidth, -4.5*\cellheight  -\yoffset) {$\dots$};
\draw[thick] (5*\cellwidth, -4*\cellheight - \yoffset) rectangle (5*\cellwidth + \cellwidth, -4*\cellheight - \cellheight - \yoffset);
\node[above] at (5*\cellwidth + 0.5*\cellwidth, -4*\cellheight  -\yoffset) {$y_T$};

\def\braceXOffset{5}
\draw [thick, decorate, decoration = {brace}] (6*\cellwidth + \braceXOffset,-\cellwidth) --  (6*\cellwidth+\braceXOffset+1,-5*\cellwidth);

\draw[->] (6.75*\cellwidth, -3*\cellheight) -- (8*\cellwidth, -3*\cellheight);
\draw[->]  (9.5*\cellwidth, -3*\cellheight) -- (10.5*\cellwidth, -3*\cellheight);
\draw[->] (6.5*\cellwidth + 0.5*\cellwidth, -3*\cellheight) -- (6.5*\cellwidth + 0.5*\cellwidth, -4.5*\cellheight) -- (8*\cellwidth, -4.5*\cellheight);
\draw[->]  (9.5*\cellwidth, -4.5*\cellheight) --  (10.5*\cellwidth, -4.5*\cellheight);
\draw[->] (6.5*\cellwidth + 0.5*\cellwidth, -3*\cellheight) -- (6.5*\cellwidth + 0.5*\cellwidth, -1.5*\cellheight) -- (10*\cellwidth, -1.5*\cellheight);

\def\fex{10.5*\cellwidth};
\def\fey{-1*\cellheight};
\def\fewidth{6*\cellwidth};
\def\feheight{-1*\cellheight};
\draw[thick, align=center] (\fex, \fey) rectangle ++(\fewidth, \feheight) node[pos=0.5] {Feature extraction};
\draw[->] (\fex+\fewidth + 0.5*\cellwidth, \fey - 0.5*\cellwidth) -- (\fex + 11.5*\cellwidth, \fey - 0.5*\cellwidth);

\def\circleradius{0.25cm};
\draw[thick] (\fex - 1.75*\cellwidth, \fey - 2.5*\cellheight + 0.5*\cellheight) circle (\circleradius) node[] {$\fint$};
\draw[thick] (\fex - 1.75*\cellwidth, \fey - 4*\cellheight + 0.5*\cellheight) circle (\circleradius) node[] {$\fcomp$};

\foreach \col in {0, 1, 2, 3, 4} {
	\draw[thick, pattern=crosshatch] (\fex + \col*\cellwidth + 0.5*\cellwidth, \fey -2.5*\cellheight) rectangle ++(\cellwidth, \cellheight);
	\draw[thick, pattern=crosshatch dots] (\fex + \col*\cellwidth+ 0.5*\cellwidth, \fey -4*\cellheight) rectangle ++(\cellwidth, \cellheight);
}
\draw[->] (\fex + 5.75*\cellwidth, \fey - 2 * \cellheight) -- (\fex + 11.5*\cellwidth, \fey - 2 * \cellheight);
\draw[->] (\fex + 5.75*\cellwidth, \fey - 3.5 * \cellheight) -- (\fex + 11.5*\cellwidth, \fey - 3.5 * \cellheight);

\draw[thick, align=center] (\fex + 5.5*\cellwidth, \fey - 7*\cellheight - 0.5*\cellheight) rectangle ++(4.5*\cellwidth, 2*\cellheight) node[pos=0.5, text width=3*\cellwidth] {{Optimal} {Selection}};
\draw[->] (6.5*\cellwidth, -7.5*\cellheight) -- (15.5*\cellwidth, -7.5*\cellheight);
\draw[->] (\fex + 8.5*\cellwidth, \fey - 2 * \cellheight) -- (\fex + 8.5*\cellwidth, \fey - 5 * \cellheight);
\draw[->] (\fex + 7.5*\cellwidth, \fey - 3.5 * \cellheight) -- (\fex + 7.5*\cellwidth, \fey - 5 * \cellheight);
\draw[fill] (\fex + 8.5*\cellwidth, \fey - 2 * \cellheight) circle (0.1*\cellwidth);
\draw[fill]  (\fex + 7.5*\cellwidth, \fey - 3.5 * \cellheight)	 circle (0.1*\cellwidth);

\foreach \col in {1, 2} {
	\draw[thick] (\fex + 5.5*\cellwidth + \col*\cellwidth - \cellwidth, \fey - 10*\cellheight) rectangle ++(\cellwidth, \cellheight) node[pos=0.5] {$s_{\col}$};
}
\node[right] at  (\fex + 5.5*\cellwidth + 2.15*\cellwidth, \fey - 9.5*\cellheight) {$\dots$};
\draw[thick] (\fex + 6*\cellwidth + 3*\cellwidth, \fey - 10*\cellheight) rectangle ++(\cellwidth, \cellheight) node[pos=0.5] {$s_{T}$};
\draw[->] (\fex + 7.75*\cellwidth, \fey - 7.75*\cellheight) -- (\fex + 7.75*\cellwidth, \fey - 8.25*\cellheight);
\draw[->] (\fex + 10.75*\cellwidth, \fey - 9.5*\cellheight) -- (\fex + 11.5*\cellwidth, \fey - 9.5*\cellheight);

\draw[dashed] (\fex + 5*\cellwidth, \fey - 10.5*\cellheight) rectangle ++(5.5*\cellwidth, 2*\cellheight);
\node at (\fex + 4*\cellwidth, \fey - 9.5*\cellheight) {$\bm{s}^*$};

\def\cx{\fex + 12*\cellwidth};
\def\cy{-\cellheight};
\def\cwidth{3.5*\cellwidth};
\def\cheight{-10.5*\cellheight};
\draw[thick] (\cx, \cy) rectangle ++(\cwidth, \cheight) node[pos=0.5] {Classifier};

\draw[text width=0.75*\cwidth, align=center, dashed] (\cx+0.125*\cwidth, \cy+0.66*\cheight) rectangle ++(0.75*\cwidth, 0.3*\cheight) node[pos=0.5] {Balance Labels};

\end{tikzpicture}
    \caption{Schematic overview of our proposed method \ourmethod: Given a set of windowed \TS $\{\bm{x}_1, \dots, \bm{x}_T \}$ with the corresponding forecasts $\{ y_1, \dots, y_T \}$ we compute both models predictions $\fint(\bxt)$ and $\fcomp(\bxt)$, as well as some additional features. The predictions, along with the label, are used to predict the optimal model selection $\bm{s}^*$, which we use as labels to estimate a classifier. }
    \label{fig:overview}
\end{figure*}

\section{Related Work}

Multiple approaches for online model selection have been proposed in recent years.
Online model selection refers to choosing a model (or  form an ensemble) from a pool of pretrained models for each time point to forecast.
These strategies are often based on meta-learning where a weighting between all pretrained models is computed, and the ensemble or single model is chosen based on these weights.
Some of these methods are model-agnostic, meaning the selection and ensemble strategies do not assume a specific family of pretrained models \cite{cerqueiraArbitratedEnsembleTime2017,cerqueiraDynamicHeterogeneousEnsembles2017,saadallahOnlineExplainableModel2023,priebeDynamicModelSelection2019}.
Most of them, however, focus on maximizing predictive performance of the selection, ignoring other aspects such as interpretability of the results.
In \cite{saadallahOnlineExplainableModel2023,priebeDynamicModelSelection2019} explainability aspects are discussed since the model selection paradigm is based on the \emph{Region of Competence (RoC)} concept.
A Region of Competence of a model is defined as a subset of data on which the model performed superior to all other models in the pool.
During inference, a new input window to forecast is compared to each model's RoC and the closest model in terms of some distance measure is chosen to forecast.
In \cite{saadallahExplainableOnlineDeep2021,saadallahExplainableOnlineEnsemble2022,jakobsExplainableAdaptiveTreebased2023}, the authors propose model-specific online selection and ensembling methods with a focus on Explainability.
First, Regions of Competence are computed for each model.
Then, these Regions are refined using model-specific Explainability methods, namely GradCAM \cite{selvarajuGradCAMVisualExplanations2019} and TreeSHAP \cite{lundbergLocalExplanationsGlobal2020}.
In addition, \cite{jakobsExplainableAdaptiveTreebased2023} utilize pools of tree-based models to select from, increasing the interpretability of the overall selection approach because tree-based model predictions can be interpretable.
While these methods can give an explanation to why a model was chosen (based on the closest observed input in the Regions of Competence), the inner workings of the base models remain too complex for investigation, sometimes even if tree-based models are used.
This is why we argue that (a pool of) interpretable models could increase trust and the number of possible application scenarios for online model selection.

Another line of work considers choosing a single model to forecast an entire time series instead of the previously mentioned online approach
\cite{zhangForecastingAgriculturalCommodity2020,prudencioMetalearningApproachesSelecting2004,talagalaMetalearningHowForecast2023}.
These works also utilize meta-learning strategies, where classifier or weighting schemes are fitted on meta-features of the data and past performances for each dataset.
However, we do not consider these approaches comparable to our methodology since we are interested in online selection to maximize both interpretability as well as predictive performance.

In Hybrid Model Selection \cite{hajirahimiHybridStructuresTime2019}, many approaches first apply a linear model to the \TS, followed by a non-linear model on the residuals.
For example, in \cite{zhangTimeSeriesForecasting2003} the authors first predict with ARIMA, followed by a simple Neural Network on the (assumed to be) non-linear residuals.
In \cite{smylHybridMethodExponential2020}, the authors combine Exponential Smoothing with LSTM Neural Networks to improve prediction quality, winning the M4 Competition in the process.
However, the forecasting process on the residuals is still done by opaque black-box methods, which makes them hard to deploy in safety-critical applications.

Some works do consider other factors to be important besides predictive error.
In \cite{brehler2023combiningDecisionTreeCNN,dagheroTwostageHumanActivity2022,buschjagerRejectionEnsemblesOnline2024}, the authors are concerned with reducing energy consumption by considering smaller models (in their case, Decision Trees) that are quick and energy efficient to evaluate over larger methods (in their case, Convolutional Neural Networks), that need to be run on special, high power hardware.
If a Decision Tree predicts a class on which it is deemed to be not good enough the Convolutional Neural Network is queried instead, increasing the energy consumption.
However, since lowering energy consumption is a focus of their work the models are only evaluated if they are chosen, whereas we also incorporate all model predictions into the model selection process.

\section{Methodology}
\label{sec:methodology}

Let $\bm{X} \in \mathbb{R}^{T}$ be a univariate \TS  of length $T$.
We denote with $\X{t}$ the value of $\bm{X}$ at time $t$.
The goal of forecasting is then to predict the value $\X{t+H}$ from the last $L$ known lagged values $(\X{t-L+1}, \X{t-L+2}, \dots, \X{t})$.
For ease of notation we will define the use of intervals in the superscript, i.e., $\bX{t-L+1:t} := (\X{t-L+1}, \X{t-L+2}, \dots, \X{t})$.

Let $\fint, \fcomp$ be two trained forecasters, mapping from $\mathbb{R}^{L}$ to $\mathbb{R}$ and let $[T] := \{1, \dots, T \}$.
We are interested in predicting, for each $t \in [T]$, which forecaster to choose to perform the prediction.
We are interested in two goals for our selection strategy:
\begin{enumerate}
    \item Achieving the lowest average prediction error over all $t \in [T]$
    \item Choosing one of the forecasters (in our case $\fint$) as often as possible.
\end{enumerate}
Let $\bm{x}_t := \X{t-L+1:t}$ and $y_t := \X{t+1}$.
Then, the overall observed squared error if we allow to choose between $\fint$ and $\fcomp$ for each $t \in [T]$ is given by
\begin{align*}
    \mathcal{L}(\bm{s}) = \sum_{t=1}^T \left(y_t-s_t \fint(\bm{x}_t) - (1-s_t)\fcomp(\bm{x}_t)\right)^2, \quad s_t \in \{0, 1 \}
\end{align*}

where $\bm{s} = (s_1, \dots, s_T)^\intercal$.
In order to choose $\fint$ more often than $\fcomp$ we introduce a constraint to select $\fint$ at least $B$ times out of $T$.
This leads to the following constraint optimization problem:

\begin{align}
\label{eq:optim}
\bm{s}^* &= \argmin_{\bm{s}} ~ \mathcal{L}(\bm{s}) \quad \textrm{s.t.} \quad \bm{s} \in \{0, 1\}^T \wedge \lVert \bm{s} \rVert \geq B
\end{align}

\begin{proposition}
    \label{prop:optimal_solution}
    The optimal solution to \cref{eq:optim} is the vector $\bm{s}^* = (s_1^*, \dots, s_T^*)$ with elements
    $$
    s^*_t = \begin{cases}
        1 &\text{if } \ell(t) \leq \max(0, \ell(\pi(B))) \\
        0 &\text{otherwise}
    \end{cases}
    $$
    where $\ell(t) := (\fint(\bxt)-y_t)^2 - (\fcomp(\bxt) - y_t)^2$ and $\pi: [T] \rightarrow [T]$ is a permutation satisfying $\ell(\pi(t)) \leq \ell(\pi(t+1)) ~ \forall t \in [T-1]$.
\end{proposition}
\begin{proof}
We start by reformulating the objective $\mathcal{L}(\bm{s})$ as
\begin{align*}
    \mathcal{L}(\bm{s})
    &= \sum_{t=1}^T (\underbrace{y_t - \fcomp(\bm{x}_t)}_{=:c_t} + s_t(\underbrace{\fcomp(\bm{x}_t) - \fint(\bm{x}_t))}_{=: \delta_t})^2 \nonumber \\
    &= \sum_{t=1}^T c_t^2 +2c_ts_t\delta_t + s_t^2\delta_t^2 \nonumber \\
    &= \left( \sum_{t=1}^T (y_t - \fcomp(\bxt))^2 \right) + \left( \sum_{t=1}^T s_t (2c_t\delta_t + \delta_t^2) \right)
\end{align*}

Notice that each summand of the second sum is the difference of squared errors between $\fint$ and $\fcomp$ at time $t$:
\begin{align*}
2c_t \delta_t + \delta_t^2 &= 2(y_t - \fcomp(\bxt))(\fcomp(\bxt) - \fint(\bxt)) + (\fcomp(\bxt) - \fint(\bxt))^2 \\
&= \fint(\bxt)^2 -2\fint(\bxt)y_t + y_t^2 - (y_t^2 - 2\fcomp(\bxt)y_t + \fcomp(\bxt)^2)  \\
&= (\fint(\bxt) -y_t)^2 - (\fcomp(\bxt) - y_t)^2  \\
&= \ell(t)  \\
\end{align*}

Since the first sum in $\mathcal{L}(\bm{s})$ is independent of $\bm{s}$ the optimization problem is equivalent to
\begin{align}
\label{eq:final_optimization}
\bm{s}^* &= \argmin_{\bm{s}} \sum_{t=1}^T s_t\ell(t) \quad \textrm{s.t.} \quad \bm{s} \in \{ 0, 1 \}^T \wedge \abs{\bm{s}} \geq B.
\end{align}

With this reformulation it is easy to see that the optimal solution (ignoring the constraint $\abs{\bm{s}} \geq B$) is given by choosing all indices $t$ for which $\ell(t) \leq 0$, i.e.,
$$
s^*_t = \begin{cases} 1 &\text{if } \ell(t) \leq 0 \\ 0 &\text{otherwise} \end{cases}
$$

In fact, if $\abs{\bm{s}^*} \geq B$ we have found a constraint satisfying solution (and included as many predictions of $\fint$ as possible) and are done.
Otherwise, we know that we already chose $t \in \{ \pi(1), \dots, \pi(N) \}$ with $N = \abs{\bm{s}^*}$ and we know that $N < B$.
This means that by choosing all $t$ with $\ell(t) \leq \ell(\pi(B))$ we include the negative error differences from $\bm{s}^*$ as well as the necessary smallest differences to satisfy the constraint.
Thus, we can rewrite the optimal solution to
$$
s^*_t = \begin{cases}
    1 &\text{if } \ell(t) \leq \max(0, \ell(\pi(B))) \\
    0 &\text{otherwise}
\end{cases}
$$
\end{proof}

\begin{figure}
    \centering
    \includegraphics{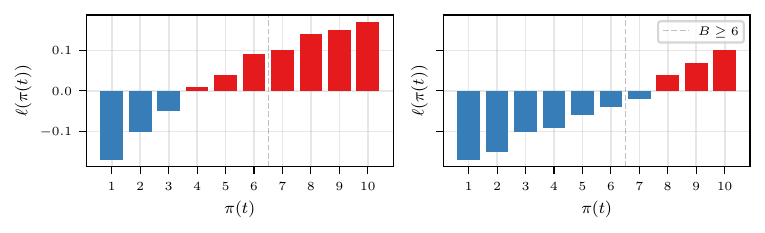}
    \caption{Illustrative example of how many (and which) entries to choose to get the optimal, constraint-satisfying solution. On the left we have to choose the indices $\pi(4), \pi(5)$ and $\pi(6)$ even though their corresponding loss difference $\ell(t)$ is positive to satisfy the constraint $B \geq 6$. On the right side we see that after choosing the first $6$ entries we can continue and further decrease the loss, choosing $7$ entries in total.}
    \label{fig:optimum_example}
\end{figure}

Intuitively, \cref{prop:optimal_solution} states that the optimal solution is given by choosing $\fint$ whenever it outperforms $\fcomp$.
In addition, we (might) need to choose $\fint$ even if the total error  increases until the constraint is satisfied.
Both scenarios are visualized with an example in \cref{fig:optimum_example}.
For both examples we have $B\geq6$, i.e., we have to choose $\fint$ at least $6$ times.
On the left side the optimal solution is given by choosing $B$ exactly $6$ times since $\ell(\pi(t)) > 0$ already for $t > 3$.
On the right side however we choose $\fint$ a total of $7$ times to get the optimal solution.

While the framework is general and can be adapted to many use-cases we want to define the constraint $B$ as an \textit{interpretability constraint} by choosing for $\fint$ some small, interpretable model and for $\fcomp$ some high-performance black-box model.
That way, we can see the problem of finding $\bm{s}^*$ as finding the best possible loss when selecting between $\fint$ and $\fcomp$ while being as interpretable as possible given a user-specific $B$.

However, we obviously do not have access to all necessary information (specifically to $y_t$) in practice to compute the optimal selection on unseen data.
Thus, we propose to interpret each $s_t$ as the output of a binary classifier.
To train this classifier, we can use data not used to train $\fint$ and $\fcomp$ and for which all necessary information is known to compute $\bm{s}^*$ and use this as the label.
The process is also visualized schematically in \cref{fig:overview}.

\section{Experiments}
In this section, we aim to answer the following research questions:

\begin{itemize}
    \item \textbf{Q1:} What is the impact of large interpretability constraints $B$ on predictive performance?
    \item \textbf{Q2:} How accurate can we estimate the true optimal selection for different $B$?
    \item \textbf{Q3:} How does our method perform in terms of prediction error and interpretable model selection against state-of-the-art online model selection methods?
\end{itemize}

\subsection{Experimental Setup}
\label{sec:experimentalsetup}
We utilize $6$ univariate \TS datasets of different application domains and with varying characteristics for our experiments (see \cref{tab:datasets} for a overview of the datasets and their characteristics).
These datasets are provided by the Monash Forecasting Repository \cite{godahewaMonashTimeSeries2021}.
We process each \TS by using the first $80\%$ as training data, the following $10\%$ as validation data and the final $10\%$ as test data.
To only keep meaningful \TS we discard series where all timesteps in the training, validation or test data are constant.
Since the \TS frequency of the datasets vary we chose to also vary the amount of lagged values accordingly.
We chose to use one day ($L=48$) for the $30$ minute datasets, $14$ days for daily and $24$ hours for the hourly datasets.
We set $H=1$ for all experiments.
The code to reproduce all of the following experiments is publicly available on our Github\footnote{\url{https://github.com/MatthiasJakobs/aalf/tree/dami}}.

\begin{table}
    \centering
    \setlength\tabcolsep{0pt}
\begin{tabular*}{\linewidth}{@{\extracolsep{\fill}} llllll}
\toprule
Name & Datapoints & Length (min/max/avg) & Freq. & $L$ & Type \\
\midrule
Aus. Electricity Demand & 5 & 230,734/232,271/231,051.20 & 30m & 48 & Energy\\
NN5 (Daily) & 111 & 790/790/790.00 & 1d & 14 & Banking\\
Weather & 2997 & 1,331/59,040/14,263.00 & 1d & 14 & Nature\\
Pedestrian Counts & 66 & 574/96,423/47,458.39 & 1h & 24 & Transport\\
KDD Cup & 255 & 9,503/10,920/10,895.24 & 1h & 24 & Nature\\
Solar & 137 & 8,760/8,760/8,760.0 & 1h & 24 & Energy\\
\bottomrule
\end{tabular*}
    \caption{Properties of the used datasets. All datasets are provided by the Monash Forecasting Repository \cite{godahewaMonashTimeSeries2021}. We resampled the Solar dataset to a frequency of one hour from 10 minutes to decrease computational demand.}
    \label{tab:datasets}
\end{table}

\subsection{Training the base models}
In order to investigate the impact of our model selection method on predictive performance and interpretability we have to train both $\fint$ and $\fcomp$.
To make $\fint$ as simple and interpretable as possible we opted for a $\text{AR}$ autoregressive model \cite{hyndmanForecastingPrinciplesPractices2021}.
$\text{AR}$ is a linear model given by
$$
\text{AR}(\bX{t-L:t-1}) = \bm{\phi}^T\bX{t-L:t-1} + \varepsilon_t, \quad \varepsilon_t \sim N(0,1)
$$
and also (as part of the ARIMA model) a very competitive baseline, even for current \TS forecasting methods.
Moreover, it is a well-understood and well-studied forecasting method that is fairly interpretable.
Since $\text{AR}$ is a local model we trained one model per \TS, more specifically on the training split of each \TS.
For $\fcomp$ we initially trained three different models.
The first one is a simple multi-layer neural network with two hidden layers (FCNN) with ReLU activations.
The second model is a Convolutional Neural Network (CNN) with two layers of convolutions, followed by one prediction layer.
Model three is DeepAR \cite{salinasDeepARProbabilisticForecasting2019}, which is a competitive state-of-the-art \TS forecasting model.
All three models are \textit{global} forecasters, meaning that we trained one model per dataset.
Specifically, they were trained on the combined training of all \TS in a dataset and afterwards evaluated on the validation and test parts of each \TS individually.

\begin{table}
    \centering
    \begin{tabular}{p{2.5cm}p{0.75cm}p{0.85cm}p{0.75cm}p{0.85cm}p{0.75cm}p{0.85cm}p{0.75cm}p{0.85cm}}
\toprule
 & \multicolumn{2}{l}{AR} & \multicolumn{2}{l}{FCNN} & \multicolumn{2}{l}{DeepAR} & \multicolumn{2}{l}{CNN} \\
 & RMSE & SMAPE & RMSE & SMAPE & RMSE & SMAPE & RMSE & SMAPE \\
\midrule
Aus. Electricity Demand & 0.086 & 0.181 & \textbf{0.063} & 0.151 & 0.069 & \textbf{0.139} & 0.069 & 0.152 \\
NN5 (Daily) & 0.735 & 0.836 & 0.688 & \textbf{0.758} & 0.705 & 0.786 & \textbf{0.685} & 0.762 \\
Weather & 0.673 & 0.859 & 0.667 & 0.813 & \textbf{0.665} & 0.811 & 0.670 & \textbf{0.801} \\
Pedestrian Counts & 0.287 & 0.366 & 0.245 & 0.312 & \textbf{0.232} & \textbf{0.290} & 0.278 & 0.347 \\
KDD Cup & \textbf{0.362} & 0.498 & 0.367 & 0.492 & 0.416 & \textbf{0.489} & 0.367 & 0.493 \\
Solar & 0.220 & 0.310 & 0.202 & 0.272 & 0.250 & 0.318 & \textbf{0.193} & \textbf{0.264} \\
\bottomrule
\end{tabular}

    \caption{The averaged root mean squared error (RMSE) and symmetric mean absolute percentage error (SMAPE) values for all base models over all datasets. The average was computed over each \TS per dataset. Notice that we highlighted the model with the smallest error (per dataset and metric) in bold face.}
    \label{tab:single_results}
\end{table}
We evaluated all models on the test portion of each dataset using two metrics.
First, we compute the root mean squared error (RMSE), given by
$$
\text{RMSE}(\hat{\bm{y}}, \bm{y}^*) = \sqrt{\sum_{i=1}^N (\hat{y}^{(i)} - y^{*(i)})^2}
$$
where $\hat{\bm{y}}$ are the predicted values for some forecaster and $\bm{y}^*$ is the ground truth.
The second metric is the symmetric mean absolute percentage error (SMAPE), given by
$$
\text{SMAPE}(\hat{\bm{y}}, \bm{y}^*) = \frac{1}{N} \sum_{i=1}^N \frac{2\lvert y^{*(t)} - \hat{y}^{(t)}\rvert}{\lvert y^{*(t)}\rvert + \lvert\hat{y}^{(t)}\rvert}
$$
In order to produce one aggregated value of RMSE and SMAPE per dataset we averaged the individual values computed on the test portions of each \TS.
We present these results in \cref{tab:single_results}.
With exception of the KDD Cup dataset we find that one of the larger, global base models always outperforms the AR model, both in terms of RMSE as well as SMAPE.
For the NN5 (Daily) and Solar datasets we see that the CNN performs best in terms or RMSE error, while DeepAR performs best (also in terms of RMSE) on the Weather and Pedestrian Counts dataset.
Finally, the FCNN outperforms the other models on the Australian Electricity Demand dataset.
However, we can see that the SMAPE value distribution is not necessarily correlated with RMSE.
\begin{figure}[h]
    \centering
    \begin{tikzpicture}[
  group line/.style={semithick},
]

\begin{axis}[
  clip={false},
  grid={both},
  axis line style={draw=none},
  tick style={draw=none},
  xticklabel pos={upper},
  y dir={reverse},
  xmin={0.5},
  ymin={0.66},
  legend style={at={(0.98, 0.7)}, font=\fontsize{8}{8}\selectfont,/tikz/every odd column/.append style={column sep=.25em}},
  legend cell align={left},
  title style={yshift=\baselineskip},
  width={360},
  ytick={1,2,3,4,5,6},
  yticklabels={{Aus.\\Electricity\\Demand},{NN5\\(Daily)},{Weather},{Pedestrian\\Counts},{KDD\\Cup},{Solar}},
  xmax={4.5},
  ymax={6.66},
  height={0.9*\axisdefaultheight},
  cycle list={{color1,mark=x,very thick,mark options={scale=2}},{color3,mark=x,very thick,mark options={scale=2}},{color2,mark=x,very thick,mark options={scale=2}},{color4,mark=x,very thick,mark options={scale=2}}},
  xticklabel style={font=\fontsize{8}{8}\selectfont},
  yticklabel style={font=\fontsize{8}{8}\selectfont, align=right},
]

\addplot+[only marks] coordinates {
  (3.8, 1)
  (3.4594594594594597, 2)
  (3.2248915582248916, 3)
  (3.0757575757575757, 4)
  (2.2509803921568627, 5)
  (2.9927007299270074, 6)
};
\addlegendentry{AR}
\addplot+[only marks] coordinates {
  (1.8, 1)
  (1.8918918918918919, 2)
  (2.224557891224558, 3)
  (2.106060606060606, 4)
  (2.4823529411764707, 5)
  (1.9416058394160585, 6)
};
\addlegendentry{FCNN}
\addplot+[only marks] coordinates {
  (2.0, 1)
  (2.7387387387387387, 2)
  (1.712379045712379, 3)
  (1.606060606060606, 4)
  (2.8, 5)
  (3.9927007299270074, 6)
};
\addlegendentry{DeepAR}
\addplot+[only marks] coordinates {
  (2.4, 1)
  (1.90990990990991, 2)
  (2.8381715048381717, 3)
  (3.212121212121212, 4)
  (2.466666666666667, 5)
  (1.072992700729927, 6)
};
\addlegendentry{CNN}
\draw[group line] (axis cs:1.8,1.2832618025751072) -- ++(0pt,-3pt) -- ([yshift=-3pt]axis cs:3.8,1.2832618025751072) -- ++(0pt,3pt);
\draw[group line] (axis cs:1.8918918918918919,2.2832618025751072) -- ++(0pt,-3pt) -- ([yshift=-3pt]axis cs:1.90990990990991,2.2832618025751072) -- ++(0pt,3pt);
\draw[group line] (axis cs:3.0757575757575757,4.283261802575107) -- ++(0pt,-3pt) -- ([yshift=-3pt]axis cs:3.212121212121212,4.283261802575107) -- ++(0pt,3pt);
\draw[group line] (axis cs:2.2509803921568627,5.283261802575107) -- ++(0pt,-3pt) -- ([yshift=-3pt]axis cs:2.8,5.283261802575107) -- ++(0pt,3pt);

\end{axis}
\end{tikzpicture}
    \caption{Critical Difference diagrams for all 6 datasets (each row corresponds to one dataset). The x axis shows the average rank of each model over the \TS in each dataset. If two models are not found to have significantly different average ranks based on a Wilcoxon signed-rank test they are connected with a horizontal bar.}
    \label{fig:cdd_single_models}
\end{figure}

To investigate this difference and to gain more insight into the individual model performances we analyzed how significant the performance differences were using a Critical Difference (CD) diagram \cite{demsar2006statistical,benavoli2016should}.
A CD diagram visualizes (for multiple treatments and multiple outcomes) whether or not the treatments are significantly different from each other or not.
In our case, each model (i.e., AR, FCNN, CNN and DeepAR) is a treatment and the RMSE loss on each \TS is the outcome.
First, the average rank of each model (considerung the model error) is computed.
Then, a Friedman test is conducted to assert whether the differences between average ranks are significant at all.
If so,  we proceed by testing pairwise significance using a Wilcoxon signed-rank test.
In \cref{fig:cdd_single_models} we present the results of these tests for each dataset.
Each row corresponds to one datasets CD diagram, with the average model rank (computed over all \TS in the dataset) shown on the x axis.
If the Wilcoxon test cannot distinguish two models a horizontal bar connects them, signifying that they are not significantly different enough (with significance level $\alpha=0.05$).

After the analysis of the base models we now select one of the black-box models as $\fcomp$ for each dataset (since our methods assumes a binary selection as outlined in \cref{sec:methodology}).
Specifically, based on the results in \cref{fig:cdd_single_models} and \cref{tab:single_results} we choose DeepAR for the Weather and Pedestrian Counts datasets, the CNN for the NN5 (Daily) and Solar dataset as well as the FCNN for the Australian Electricity Demand and KDD Cup datasets.

\subsection{Impact of the interpretability constraint}
\label{sec:q1}
\begin{figure}[h]
    \centering
    \includegraphics{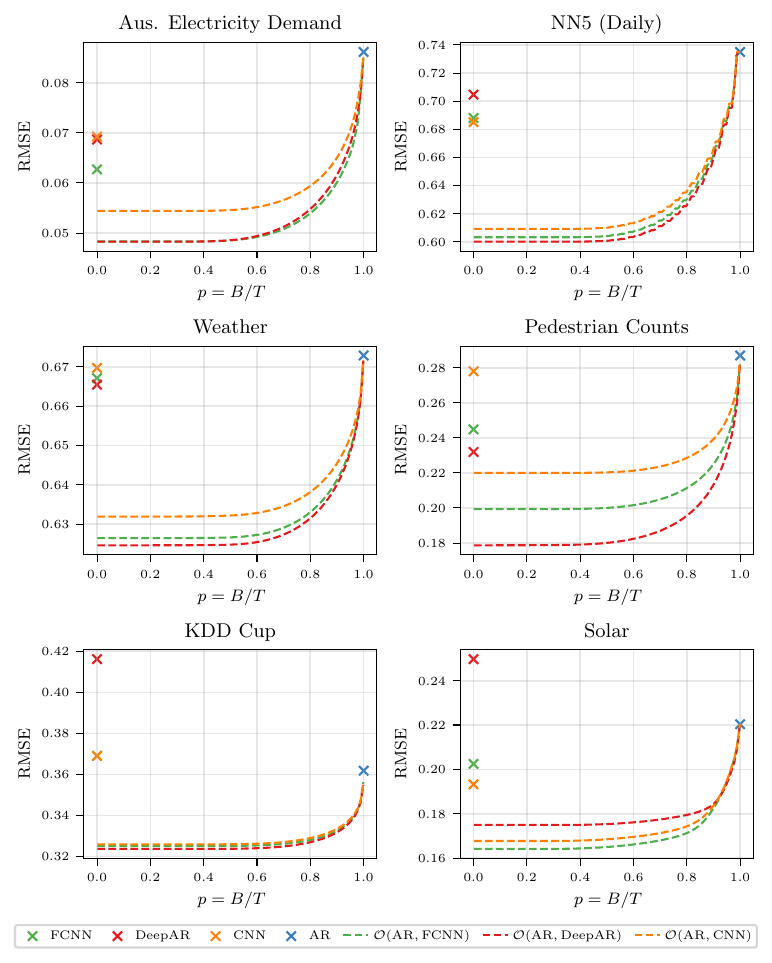}
    \caption{The optimal selection $\mathcal{O}(\fint, \fcomp)$ between pairs of models (shown as lines) versus the performance of the individual models. The y axis corresponds to average RMSE over the datasets while the x axis corresponds to how often $\fint$ is chosen over $\fcomp$.  Note that $\mathcal{O}$ is computed using ground-truth data and thus represents the best possible loss achievable for a given $p = B/T$.}
    \label{fig:loss_floor}
\end{figure}

In this section we will discuss the impact of the interpretability constraint $B$ on predictive performance.
To do that we will investigate the RMSE of combining $\fint$ and $\fcomp$ if we assume the optimal solution as given by \cref{prop:optimal_solution}.
Formally, we define $\mathcal{O}(\fint, \fcomp)$ to be the optimal selector according to \cref{prop:optimal_solution} for the models $\fint$ and $\fcomp$.
Further, since the length of the \TS in each dataset can vary drastically we define the relative interpretability constraint $p = B / T$.
In \cref{fig:loss_floor} we show (for all datasets and combinations of $\fint$ and $\fcomp$) how good the optimal selector would perform.
To do that we generated $100$ samples of $p \in (0, 1)$ and plot the RMSE (y axis) versus $p$ (x axis) for different combinations of $\fint$ and $\fcomp$.

First, notice that in order to achieve comparable loss to $\fcomp$ the optimal selector rarely needs to use $\fcomp$ more than $10$ to $15$ percent of the cases (i.e., we achieve comparable performance at $p=0.85$ to $p=0.9$).
This suggests that little help from $\fcomp$ would be needed to drastically increase the performance of $\fint$.
The notable exception to this is the KDD Cup dataset since (as mentioned earlier) $\fint$ is performing very well on this dataset.
Second, notice that for roughly $p < 0.6$ the optimal selector stops improving in terms of error.
This is due to the fact that in \cref{eq:optim} the constraint is stated for \textit{at least} $B/T$ predictions.

To summarize, our experiments suggest that (given a perfect model selection mechanism) it is possible to achieve a low RMSE error with very high interpretability constraints, resulting in an overall more interpretable \TS forecasting process. This answers research question \textbf{Q1}.

\subsection{Estimating the optimal solution}
\label{sec:q2}
\begin{table}[h]
	\centering
	\begin{tabular}{llllllll}
\toprule
 &  & $p=0.5$ & $p=0.6$ & $p=0.7$ & $p=0.8$ & $p=0.9$ & $p=0.95$ \\
\midrule
\multirow[t]{5}{*}{Aus. Electricity Demand} & $\mathtt{RND}$ & 0.501 & 0.600 & 0.700 & 0.800 & 0.900 & 0.950 \\
 & $\mathtt{LR}$ & 0.620 & 0.751 & 0.834 & 0.901 & 0.956 & 0.980 \\
 & $\mathtt{SVM}$ & 0.653 & 0.769 & 0.842 & 0.903 & 0.954 & 0.978 \\
 & $\mathtt{RF}$ & \underline{0.700} & \underline{0.799} & \underline{0.867} & \underline{0.920} & \underline{0.965} & \textbf{0.984} \\
 & $\mathtt{RFu}$ & \textbf{0.719} & \textbf{0.804} & \textbf{0.868} & \textbf{0.920} & \textbf{0.965} & \underline{0.983} \\
\midrule
\multirow[t]{5}{*}{NN5 (Daily)} & $\mathtt{RND}$ & 0.516 & 0.599 & 0.708 & 0.801 & 0.907 & 0.954 \\
 & $\mathtt{LR}$ & 0.578 & 0.684 & 0.795 & 0.869 & 0.943 & 0.973 \\
 & $\mathtt{SVM}$ & \underline{0.587} & \textbf{0.734} & \textbf{0.828} & \textbf{0.889} & \textbf{0.952} & 0.976 \\
 & $\mathtt{RF}$ & 0.585 & 0.716 & \underline{0.819} & \underline{0.884} & \underline{0.951} & \underline{0.976} \\
 & $\mathtt{RFu}$ & \textbf{0.637} & \underline{0.724} & 0.812 & 0.877 & 0.949 & \textbf{0.976} \\
\midrule
\multirow[t]{5}{*}{Weather} & $\mathtt{RND}$ & 0.511 & 0.602 & 0.701 & 0.801 & 0.901 & 0.951 \\
 & $\mathtt{LR}$ & 0.619 & 0.757 & 0.825 & 0.890 & 0.947 & 0.974 \\
 & $\mathtt{SVM}$ & 0.595 & 0.745 & 0.824 & 0.890 & \underline{0.948} & \underline{0.974} \\
 & $\mathtt{RF}$ & \underline{0.651} & \underline{0.759} & \textbf{0.837} & \textbf{0.898} & \textbf{0.950} & \textbf{0.974} \\
 & $\mathtt{RFu}$ & \textbf{0.655} & \textbf{0.768} & \underline{0.833} & \underline{0.892} & 0.946 & 0.972 \\
\midrule
\multirow[t]{5}{*}{Pedestrian Counts} & $\mathtt{RND}$ & 0.501 & 0.600 & 0.700 & 0.800 & 0.900 & 0.950 \\
 & $\mathtt{LR}$ & 0.636 & 0.757 & 0.834 & 0.898 & 0.952 & 0.976 \\
 & $\mathtt{SVM}$ & 0.657 & 0.769 & 0.849 & 0.908 & 0.957 & 0.979 \\
 & $\mathtt{RF}$ & \textbf{0.728} & \underline{0.809} & \underline{0.868} & \underline{0.918} & \underline{0.961} & \underline{0.981} \\
 & $\mathtt{RFu}$ & \underline{0.723} & \textbf{0.819} & \textbf{0.875} & \textbf{0.922} & \textbf{0.962} & \textbf{0.981} \\
\midrule
\multirow[t]{5}{*}{KDD Cup} & $\mathtt{RND}$ & 0.531 & 0.601 & 0.700 & 0.800 & 0.901 & 0.950 \\
 & $\mathtt{LR}$ & 0.568 & 0.685 & \underline{0.785} & 0.863 & 0.933 & 0.965 \\
 & $\mathtt{SVM}$ & \underline{0.595} & \textbf{0.709} & \textbf{0.806} & \textbf{0.883} & \textbf{0.947} & \textbf{0.974} \\
 & $\mathtt{RF}$ & 0.579 & 0.674 & 0.782 & \underline{0.865} & 0.933 & 0.964 \\
 & $\mathtt{RFu}$ & \textbf{0.643} & \underline{0.705} & 0.780 & 0.857 & \underline{0.933} & \underline{0.967} \\
\midrule
\multirow[t]{5}{*}{Solar} & $\mathtt{RND}$ & 0.500 & 0.601 & 0.700 & 0.801 & 0.900 & 0.950 \\
 & $\mathtt{LR}$ & 0.642 & 0.758 & 0.833 & 0.900 & 0.949 & 0.975 \\
 & $\mathtt{SVM}$ & 0.712 & 0.808 & 0.852 & 0.907 & 0.948 & 0.975 \\
 & $\mathtt{RF}$ & \textbf{0.784} & \underline{0.832} & \textbf{0.871} & \textbf{0.919} & \textbf{0.961} & \textbf{0.978} \\
 & $\mathtt{RFu}$ & \underline{0.759} & \textbf{0.835} & \underline{0.865} & \underline{0.911} & \underline{0.960} & \underline{0.976} \\
\bottomrule
\end{tabular}

	\caption{The $F_1$ scores of all classifiers for every datasets and different settings of $p$. The best value per dataset and value of $p$ are shown in bold face, while the second best value is shown with an underline.}
	\label{tab:classifier_table}
\end{table}
After the analysis of optimal selection in the previous section we now aim to learn to select from historical data.
Specifically, we compute optimal selection (according to \cref{prop:optimal_solution}) on the validation data splits defined in \cref{sec:experimentalsetup} and use this as the label in our binary classification task.
Let $(\bxt, y_t)$ be the current window and label.
Then, we compute the optimal selection $s_t^*$ using $\fint(\bxt)$, $\fcomp(\bxt)$ and $y_t$.
Furthermore, we generate the following features to help in fitting our classifier.
$ \Delta_p(t) = \fint(\bxt) - \fcomp(\bxt)$ is the difference in prediction at time $t$.
$\Delta_e(t) =  (\fint(\bm{x}_{t-1}) - y_{t-1})^2 - (\fcomp(\bm{x}_{t-1}) - y_{t-1})^2$ is the last known difference in observed errors between the models.
Lastly, let $\Sigma(t) = (L^{-1}\sum_{i=1}^L x^{(i)}_{t}, \min(\bxt), \max(\bxt))$ be the statistics of mean, min and max values over the current window $\bm{x}_t$.
Then, we have the new datapoint $(\bm{x}'_t, s_t^*)$ that we utilize to fit the classifier with
$$
\bm{x}_t' = \bxt \oplus \Delta_p(t) \oplus \Delta_e(t) \oplus \Sigma(t)
$$
where $\oplus$ is the vector concatenation.

Next, we evaluate the performance of multiple classifier from different model families on this classification task:
\begin{itemize}
	\item \texttt{RM}: A random baseline that (given a desired $p$) generates a selection $\hat{\bm{s}} \sim \mathcal{B}(T, p)$ where $\mathcal{B}$ is the binomial distribution.
	\item \texttt{LR}: A Logistic Regression, fitted with L-BFGS
	\item \texttt{RF}: A Random Forest with $128$ estimators.
	\item \texttt{SVM}: A Support Vector Machine with an RBF kernel.
\end{itemize}

To also investigate the impact of high class imbalance on the classifiers performance we also evaluated \texttt{RFu}.
This is a ensemble of Random Forests where each ensemble member is trained using a balanced version of the dataset.
The balancing is done using upsampling (with replacement) of the minority class so that both classes appear in equal proportion.
We did not find that upsampling increased performance on either Logisitc Regression or SVMs, which is why we omit it from further discussion.

The result of the experiment can be found in \cref{tab:classifier_table}.
Specifically, for each dataset and classifier combination we evaluated the $F_1$ score for different levels $p \in \{ 0.5, 0.6, 0.7, 0.8, 0.9, 0.95 \}$.
We follow the approach outlined in \cite{formanApplesapplesCrossvalidationStudies2010} to compute one $F_1$ value per dataset and classifier by computing the true positive (TP), false positive (FP) and false negative (FN) predictions per \TS and summing them over the entire dataset.
Then, one $F_1$ score is computed using the aggregated values via
$$
F_1(\mathrm{TP}, \mathrm{FN}, \mathrm{FP}) = (2\mathrm{TP}) / (2 \mathrm{TP}  + \mathrm{FP} + \mathrm{FN}).
$$

Notice that we highlighted the best $F_1$ score per dataset and value of $p$ in bold, as well as the second highest value with an underline.
As can be seen, both versions of the Random Forest (\texttt{RF} and \texttt{RFu}) perform very well over multiple datasets and values of $p$, achieving the highest or second highest $F_1$ scores fairly consistently.
Interestingly, the superior performance of \texttt{RFu} is especially noticeable for low values of $p$ (i.e., $p=0.5$ and $p=0.6$), suggesting that the balancing of class label proportions is not so important there and the ensembling of base learners is increasing the performance instead.

Due to the highest empirical performance, especially for lower values of $p$, we decide to choose \texttt{RFu} for the next experiment.
This answers research question \textbf{Q2}.

\subsection{Comparison to state-of-the-art}
\label{sec:q3}
\begin{figure}[h]
	\centering
	\begin{tikzpicture}[
  treatment line/.style={rounded corners=1.5pt, line cap=round, shorten >=1pt},
  treatment label/.style={font=\small},
  group line/.style={ultra thick},
]

\begin{axis}[
  clip={false},
  axis x line={center},
  axis y line={none},
  axis line style={-},
  xmin={1},
  ymax={0},
  scale only axis={true},
  width={300},
  ticklabel style={anchor=south, yshift=1.3*\pgfkeysvalueof{/pgfplots/major tick length}, font=\small},
  every tick/.style={draw=black},
  major tick style={yshift=.5*\pgfkeysvalueof{/pgfplots/major tick length}},
  minor tick style={yshift=.5*\pgfkeysvalueof{/pgfplots/minor tick length}},
  title style={yshift=\baselineskip},
  xmax={12},
  ymin={-9.5},
  height={8\baselineskip},
]

\draw[treatment line] ([yshift=-2pt] axis cs:3.583309997199664, 0) |- (axis cs:2.583309997199664, -2.0)
  node[treatment label, anchor=east] {$\textbf{OMS-RoC}$};
\draw[treatment line] ([yshift=-2pt] axis cs:4.179081489778773, 0) |- (axis cs:2.583309997199664, -3.0)
  node[treatment label, anchor=east] {$\textbf{AALF}_{p=0.5}$};
\draw[treatment line] ([yshift=-2pt] axis cs:4.752030243629235, 0) |- (axis cs:2.583309997199664, -4.0)
  node[treatment label, anchor=east] {$\textbf{AALF}_{p=0.6}$};
\draw[treatment line] ([yshift=-2pt] axis cs:5.09017082049846, 0) |- (axis cs:2.583309997199664, -5.0)
  node[treatment label, anchor=east] {$\textbf{DETS}$};
\draw[treatment line] ([yshift=-2pt] axis cs:5.151498179781574, 0) |- (axis cs:2.583309997199664, -6.0)
  node[treatment label, anchor=east] {$\textbf{ADE}$};
\draw[treatment line] ([yshift=-2pt] axis cs:5.190702884346122, 0) |- (axis cs:2.583309997199664, -7.0)
  node[treatment label, anchor=east] {$\textbf{AALF}_{p=0.7}$};
\draw[treatment line] ([yshift=-2pt] axis cs:5.5879305516662, 0) |- (axis cs:12.385046205544665, -7.0)
  node[treatment label, anchor=west] {$\textbf{KNN-RoC}$};
\draw[treatment line] ([yshift=-2pt] axis cs:6.194903388406609, 0) |- (axis cs:12.385046205544665, -6.0)
  node[treatment label, anchor=west] {$\textbf{AALF}_{p=0.8}$};
\draw[treatment line] ([yshift=-2pt] axis cs:7.436012321478578, 0) |- (axis cs:12.385046205544665, -5.0)
  node[treatment label, anchor=west] {$\textbf{AALF}_{p=0.9}$};
\draw[treatment line] ([yshift=-2pt] axis cs:8.121674600952113, 0) |- (axis cs:12.385046205544665, -4.0)
  node[treatment label, anchor=west] {$\textbf{AALF}_{p=0.95}$};
\draw[treatment line] ([yshift=-2pt] axis cs:11.327639316718006, 0) |- (axis cs:12.385046205544665, -3.0)
  node[treatment label, anchor=west] {$\textbf{LastValue}$};
\draw[treatment line] ([yshift=-2pt] axis cs:11.385046205544665, 0) |- (axis cs:12.385046205544665, -2.0)
  node[treatment label, anchor=west] {$\textbf{MeanValue}$};
\draw[group line] (axis cs:4.752030243629235, -2.6666666666666665) -- (axis cs:5.151498179781574, -2.6666666666666665);

\end{axis}
\end{tikzpicture}
	\caption{Critical Difference Diagram (over RMSE error) of all evaluated model selectors, computed over all datasets. The average rank is shown above (smaller is better). Two selectors are connected with a horizontal line if they did no show significantly different performance.}
	\label{tab:baseline_cdd}
\end{figure}

\begin{figure}[h]
	\centering
	\includegraphics{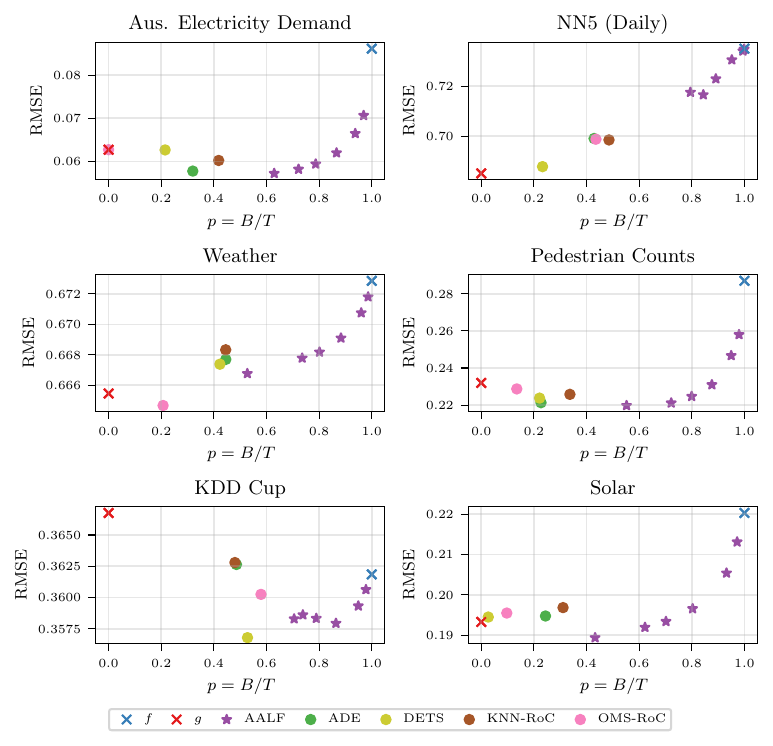}
	\caption{Comparison between average RMSE (y axis) and empirical selection of $\fint$ (x axis). \ourmethod was evaluated for multiple hyperparameter values of $p$. }
	\label{fig:scatter}
\end{figure}

Finally, we compare our framework, \ourmethod, to state-of-the-art online model selection methods, both in terms of predictive performance as well as in terms of how often $\fint$ is chosen over $\fcomp$.
\ourmethod comprises of an ensemble of Random Forest classifiers, each of which was trained on a differently sampled, balanced dataset of input features and optimal selections as outlined before.
In this experiment, we will again (as in the previous experiment) consider multiple different values of $p$, namely $p \in \{ 0.5, 0.6, 0.7, 0.8, 0.9, 0.95 \}$.

The state-of-the-art methods and baselines which we compare ourselves against are:

\begin{itemize}
	\item \textbf{ADE} \cite{cerqueiraArbitratedEnsembleTime2017} tries to learn the expected error for each model on unseen data based on measured validation error. Afterwards, a weighting scheme is proposed, assigning each model a weight which can be used for ensembling. We, however, choose the model with the highest weight at each time step for prediction.
	\item \textbf{DETS} \cite{cerqueiraDynamicHeterogeneousEnsembles2017} is an ensembling method that selects a subset of models for prediction based on recent errors. These errors are computed over a window of past predictions for each model, smoothed and aggregated into a weighting for ensembling. We again choose the model with the highest weight at each time step for prediction.
	\item \textbf{KNN-RoC} \cite{priebeDynamicModelSelection2019} utilizes the idea of Regions of Competence, where each model keeps a buffer of datapoints it performed best at during validation. At test time, the new input window is compared to all Regions of Competence and the model with the highest similarity is chosen to predict the next value.
	\item \textbf{OMS-RoC} \cite{saadallahOnlineExplainableModel2023} also utilizes Regions of Competence. However, OMS-RoC clusters the data and assigns model expertise to entire clusters, instead of to individual datapoints.
	\item \textbf{MeanValue} and \textbf{LastValue} serve as simple, yet important baselines, where the prediction is the average or last value of the window $\bm{x}_t$, respectively.
\end{itemize}

First, we will analyze the overall predictive performance.
For that, we created a Critical Difference Diagram (see \cref{tab:baseline_cdd}) over all \TS in all datasets.
Similar to the CD diagram in \cref{fig:cdd_single_models} we have the average rank of each method given and draw a horizontal bar between pairs of methods if they are not significantly different.

Unsurprisingly, the simple baselines \textbf{MeanValue} and \textbf{LastValue} are the worst performing methods.
Next, we notice that the versions of \ourmethod that have very high values of $p$ (above $p=0.8$) perform worse than the comparison methods.
The two versions $\ourmethod_{p=0.7}$ and $\ourmethod_{p=0.6}$, however, perform comparably to \textbf{DETS} and \textbf{ADE}, two popular state-of-the-art online model selection methods.
Moreover, $\ourmethod_{p=0.5}$ even significantly outperforms these two selectors, only being surpassed by \textbf{OMS-RoC}.
This experiment suggests that, depending on the amount of interpretability desired, we are able to achieve comparable performance to current state-of-the-art methods.

Next, we want to analyze the amount of selections of the model $\fint$ in tandem with the achieved performance.
For that, we computed the empirical $p = B/T$ (x axis), as well as the average RMSE (y axis) over all datasets and methods in \cref{fig:scatter}.
We show all versions of \ourmethod for all values of $p$ in violet.
As can be seen, the exponential trade-off between RMSE and $p$ seen in \cref{fig:loss_floor} can also be observed empirically for most datasets.
Moreover, we observe an overestimation of $p$ in most of the experiments, i.e., \ourmethod chooses $\fint$ more than necessary.
There is most likely more performance to be gained by tweaking the decision threshold or achieving a better fit of the \texttt{RFu} classifier.

As mentioned earlier, \textbf{OMS-RoC} significantly outperformed even our best performing selector ($\ourmethod_{p=0.5}$).
However, we want to stress that (as can be observed empirically in \cref{fig:scatter}) \textbf{OMS-RoC} selects $\fcomp$ significantly more often than \ourmethod on most datasets.
Thus, in terms of a performance-interpretability trade-off both \textbf{OMS-RoC} and \ourmethod perform on different ends of the spectrum.
These experiments answer research question \textbf{Q3}.

\subsection{Discussion}
Combining the results of the previous three experiments we are able to provide evidence that a online model selection strategy with a very high interpretability constraint could perform very well in theory (see \cref{sec:q1}).
With choosing $\fcomp$ only a fraction of the time we are able to significantly outperform $\fint$ since the performance increases exponentially with decreasing $p$ (see \cref{fig:loss_floor}).
In our experiments, we are generally able to achieve a good fit to the optimal selection with established classification methods such as Random Forests and Support Vector Machines (see \cref{sec:q2}).
Our proposed upsampling of the minority class further helps to improve the performance in some cases.
In other cases (especially for low values of $p$) the increase in performance is due to the fact that we ensembled the classifiers.
In comparison to the current state-of-the-art model selection methods (see \cref{sec:q3}) we performed competitively with a focus on being more interpretable than all other competitors.
While being interpretable is not a specific goal of these comparison methods we want to stress the fact that our multiobjective goal of high performance and good interpretability is (to the best of our knowledge) novel in online model selection for \TS forecasting.

\section{Conclusion}
In this work we presented a novel framework for online model selection with an interpretability constraint called \ourmethod: Almost Always Linear Forecasting.
It computes the optimal solution to a constraint optimization framework on held-out validation data, which is then used to train a classifier which later acts as an online model selector.
We investigated the theoretically best selection performance, as well as the empirically achieved performance on over 3500 real-world \TS, showing competitive performance against state-of-the-art online model selection methods while being more interpretable overall.

There are multiple open research directions that can be addressed in future work.
First, while we only evaluated our methodology on univariate data for simplicity there is no inherent reason why our method would not be applicable to multivariate data.
However, the question of which interpretable, multivariate \TS model to choose would need to be addressed since feature interactions in this context often lead to inherently complicated models.
Also, as we could observe from our conducted experiment there is more room for improvement in terms of fitting the optimal selection.
A better selection would lead to even smaller RMSE errors with the same (or more) interpretability.
Further, model selection strategies with guarantees of satisfying the constraint (in expectation, for example) could improve the applicability in safety-critical domains.
One possible direction might be found in reinforcement learning by analyzing past regret or by heuristically preventing the selector to break the constraint should $\fcomp$ be chosen too often in a fixed-size lookback window.

\bmhead{Acknowledgments}
This research has been funded by the Federal Ministry of Education and Research of Germany and the state of North Rhine-Westphalia as part of the Lamarr Institute for Machine Learning and Artificial Intelligence.

\bibliography{sn-bibliography}


\begin{thebibliography}{32}
\ifx \bisbn   \undefined \def \bisbn  #1{ISBN #1}\fi
\ifx \binits  \undefined \def \binits#1{#1}\fi
\ifx \bauthor  \undefined \def \bauthor#1{#1}\fi
\ifx \batitle  \undefined \def \batitle#1{#1}\fi
\ifx \bjtitle  \undefined \def \bjtitle#1{#1}\fi
\ifx \bvolume  \undefined \def \bvolume#1{\textbf{#1}}\fi
\ifx \byear  \undefined \def \byear#1{#1}\fi
\ifx \bissue  \undefined \def \bissue#1{#1}\fi
\ifx \bfpage  \undefined \def \bfpage#1{#1}\fi
\ifx \blpage  \undefined \def \blpage #1{#1}\fi
\ifx \burl  \undefined \def \burl#1{\textsf{#1}}\fi
\ifx \doiurl  \undefined \def \doiurl#1{\url{https://doi.org/#1}}\fi
\ifx \betal  \undefined \def \betal{\textit{et al.}}\fi
\ifx \binstitute  \undefined \def \binstitute#1{#1}\fi
\ifx \binstitutionaled  \undefined \def \binstitutionaled#1{#1}\fi
\ifx \bctitle  \undefined \def \bctitle#1{#1}\fi
\ifx \beditor  \undefined \def \beditor#1{#1}\fi
\ifx \bpublisher  \undefined \def \bpublisher#1{#1}\fi
\ifx \bbtitle  \undefined \def \bbtitle#1{#1}\fi
\ifx \bedition  \undefined \def \bedition#1{#1}\fi
\ifx \bseriesno  \undefined \def \bseriesno#1{#1}\fi
\ifx \blocation  \undefined \def \blocation#1{#1}\fi
\ifx \bsertitle  \undefined \def \bsertitle#1{#1}\fi
\ifx \bsnm \undefined \def \bsnm#1{#1}\fi
\ifx \bsuffix \undefined \def \bsuffix#1{#1}\fi
\ifx \bparticle \undefined \def \bparticle#1{#1}\fi
\ifx \barticle \undefined \def \barticle#1{#1}\fi
\bibcommenthead
\ifx \bconfdate \undefined \def \bconfdate #1{#1}\fi
\ifx \botherref \undefined \def \botherref #1{#1}\fi
\ifx \url \undefined \def \url#1{\textsf{#1}}\fi
\ifx \bchapter \undefined \def \bchapter#1{#1}\fi
\ifx \bbook \undefined \def \bbook#1{#1}\fi
\ifx \bcomment \undefined \def \bcomment#1{#1}\fi
\ifx \oauthor \undefined \def \oauthor#1{#1}\fi
\ifx \citeauthoryear \undefined \def \citeauthoryear#1{#1}\fi
\ifx \endbibitem  \undefined \def \endbibitem {}\fi
\ifx \bconflocation  \undefined \def \bconflocation#1{#1}\fi
\ifx \arxivurl  \undefined \def \arxivurl#1{\textsf{#1}}\fi
\csname PreBibitemsHook\endcsname

\bibitem[\protect\citeauthoryear{Godahewa
  et~al.}{2021}]{godahewaMonashTimeSeries2021}
\begin{bchapter}
\bauthor{\bsnm{Godahewa}, \binits{R.}},
\bauthor{\bsnm{Bergmeir}, \binits{C.}},
\bauthor{\bsnm{Webb}, \binits{G.I.}},
\bauthor{\bsnm{Hyndman}, \binits{R.J.}},
\bauthor{\bsnm{{Montero-Manso}}, \binits{P.}}:
\bctitle{Monash time series forecasting archive}.
In: \bbtitle{Neural Information Processing Systems Track on Datasets and
  Benchmarks}
(\byear{2021})
\end{bchapter}
\endbibitem

\bibitem[\protect\citeauthoryear{Hyndman
  et~al.}{2002}]{hyndmanStateSpaceFramework2002}
\begin{barticle}
\bauthor{\bsnm{Hyndman}, \binits{R.J.}},
\bauthor{\bsnm{Koehler}, \binits{A.B.}},
\bauthor{\bsnm{Snyder}, \binits{R.D.}},
\bauthor{\bsnm{Grose}, \binits{S.}}:
\batitle{A state space framework for automatic forecasting using exponential
  smoothing methods}.
\bjtitle{International Journal of Forecasting}
\bvolume{18}(\bissue{3}),
\bfpage{439}--\blpage{454}
(\byear{2002})
\doiurl{10.1016/S0169-2070(01)00110-8}
\end{barticle}
\endbibitem

\bibitem[\protect\citeauthoryear{Saadallah
  et~al.}{2020}]{saadallahBRIGHTDriftAwareDemand2020}
\begin{barticle}
\bauthor{\bsnm{Saadallah}, \binits{A.}},
\bauthor{\bsnm{{Moreira-Matias}}, \binits{L.}},
\bauthor{\bsnm{Sousa}, \binits{R.}},
\bauthor{\bsnm{Khiari}, \binits{J.}},
\bauthor{\bsnm{Jenelius}, \binits{E.}},
\bauthor{\bsnm{Gama}, \binits{J.}}:
\batitle{{{BRIGHT}}---{{Drift-Aware Demand Predictions}} for {{Taxi
  Networks}}}.
\bjtitle{IEEE Transactions on Knowledge and Data Engineering}
\bvolume{32}(\bissue{2}),
\bfpage{234}--\blpage{245}
(\byear{2020})
\doiurl{10.1109/TKDE.2018.2883616}
\end{barticle}
\endbibitem

\bibitem[\protect\citeauthoryear{Hyndman
  et~al.}{2015}]{hyndmanLargescaleUnusualTime2015}
\begin{bchapter}
\bauthor{\bsnm{Hyndman}, \binits{R.J.}},
\bauthor{\bsnm{Wang}, \binits{E.}},
\bauthor{\bsnm{Laptev}, \binits{N.}}:
\bctitle{Large-scale unusual time series detection}.
In: \bbtitle{2015 {{IEEE}} International Conference on Data Mining Workshop
  ({{ICDMW}})},
pp. \bfpage{1616}--\blpage{1619}
(\byear{2015}).
\doiurl{10.1109/ICDMW.2015.104}
\end{bchapter}
\endbibitem

\bibitem[\protect\citeauthoryear{Jakobs and
  Saadallah}{2023}]{jakobsExplainableAdaptiveTreebased2023}
\begin{bchapter}
\bauthor{\bsnm{Jakobs}, \binits{M.}},
\bauthor{\bsnm{Saadallah}, \binits{A.}}:
\bctitle{Explainable adaptive tree-based model selection for time-series
  forecasting}.
In: \bbtitle{2023 {{IEEE}} International Conference on Data Mining ({{ICDM}})},
pp. \bfpage{180}--\blpage{189}
(\byear{2023}).
\doiurl{10.1109/ICDM58522.2023.00027}
\end{bchapter}
\endbibitem

\bibitem[\protect\citeauthoryear{Saadallah
  et~al.}{2021}]{saadallahExplainableOnlineDeep2021}
\begin{bchapter}
\bauthor{\bsnm{Saadallah}, \binits{A.}},
\bauthor{\bsnm{Jakobs}, \binits{M.}},
\bauthor{\bsnm{Morik}, \binits{K.}}:
\bctitle{Explainable online deep neural network selection using adaptive
  saliency maps for time series forecasting}.
In: \beditor{\bsnm{Oliver}, \binits{N.}},
\beditor{\bsnm{{P{\'e}rez-Cruz}}, \binits{F.}},
\beditor{\bsnm{Kramer}, \binits{S.}},
\beditor{\bsnm{Read}, \binits{J.}},
\beditor{\bsnm{Lozano}, \binits{J.A.}} (eds.)
\bbtitle{Machine Learning and Knowledge Discovery in Databases. {{Research}}
  Track},
pp. \bfpage{404}--\blpage{420}.
\bpublisher{{Springer International Publishing}},
\blocation{{Cham}}
(\byear{2021})
\end{bchapter}
\endbibitem

\bibitem[\protect\citeauthoryear{Saadallah
  et~al.}{2022}]{saadallahExplainableOnlineEnsemble2022}
\begin{barticle}
\bauthor{\bsnm{Saadallah}, \binits{A.}},
\bauthor{\bsnm{Jakobs}, \binits{M.}},
\bauthor{\bsnm{Morik}, \binits{K.}}:
\batitle{Explainable online ensemble of deep neural network pruning for time
  series forecasting}.
\bjtitle{Machine Learning}
(\byear{2022})
\doiurl{10.1007/s10994-022-06218-4}
\end{barticle}
\endbibitem

\bibitem[\protect\citeauthoryear{Cerqueira
  et~al.}{2017}]{cerqueiraArbitratedEnsembleTime2017}
\begin{bchapter}
\bauthor{\bsnm{Cerqueira}, \binits{V.}},
\bauthor{\bsnm{Torgo}, \binits{L.}},
\bauthor{\bsnm{Pinto}, \binits{F.}},
\bauthor{\bsnm{Soares}, \binits{C.}}:
\bctitle{Arbitrated {{Ensemble}} for {{Time Series Forecasting}}}.
In: \beditor{\bsnm{Ceci}, \binits{M.}},
\beditor{\bsnm{Hollm{\'e}n}, \binits{J.}},
\beditor{\bsnm{Todorovski}, \binits{L.}},
\beditor{\bsnm{Vens}, \binits{C.}},
\beditor{\bsnm{D{\v z}eroski}, \binits{S.}} (eds.)
\bbtitle{Machine {{Learning}} and {{Knowledge Discovery}} in {{Databases}}}.
\bsertitle{Lecture {{Notes}} in {{Computer Science}}},
pp. \bfpage{478}--\blpage{494}.
\bpublisher{{Springer International Publishing}},
\blocation{{Cham}}
(\byear{2017})
\end{bchapter}
\endbibitem

\bibitem[\protect\citeauthoryear{Saadallah and
  Jakobs}{2023}]{saadallahOnlineDeepHybrid2023}
\begin{bchapter}
\bauthor{\bsnm{Saadallah}, \binits{A.}},
\bauthor{\bsnm{Jakobs}, \binits{M.}}:
\bctitle{Online {{Deep Hybrid Ensemble Learning}} for {{Time Series
  Forecasting}}}.
In: \beditor{\bsnm{Koutra}, \binits{D.}},
\beditor{\bsnm{Plant}, \binits{C.}},
\beditor{\bsnm{Gomez~Rodriguez}, \binits{M.}},
\beditor{\bsnm{Baralis}, \binits{E.}},
\beditor{\bsnm{Bonchi}, \binits{F.}} (eds.)
\bbtitle{Machine {{Learning}} and {{Knowledge Discovery}} in {{Databases}}:
  {{Research Track}}}.
\bsertitle{Lecture {{Notes}} in {{Computer Science}}},
pp. \bfpage{156}--\blpage{171}.
\bpublisher{Springer},
\blocation{Cham}
(\byear{2023})
\end{bchapter}
\endbibitem

\bibitem[\protect\citeauthoryear{Wolpert and
  Macready}{1997}]{wolpertNoFreeLunch1997}
\begin{barticle}
\bauthor{\bsnm{Wolpert}, \binits{D.H.}},
\bauthor{\bsnm{Macready}, \binits{W.G.}}:
\batitle{No free lunch theorems for optimization}.
\bjtitle{IEEE Transactions on Evolutionary Computation}
\bvolume{1}(\bissue{1}),
\bfpage{67}--\blpage{82}
(\byear{1997})
\doiurl{10.1109/4235.585893}
\end{barticle}
\endbibitem

\bibitem[\protect\citeauthoryear{Guidotti
  et~al.}{2018}]{guidottiSurveyMethodsExplaining2018}
\begin{barticle}
\bauthor{\bsnm{Guidotti}, \binits{R.}},
\bauthor{\bsnm{Monreale}, \binits{A.}},
\bauthor{\bsnm{Ruggieri}, \binits{S.}},
\bauthor{\bsnm{Turini}, \binits{F.}},
\bauthor{\bsnm{Giannotti}, \binits{F.}},
\bauthor{\bsnm{Pedreschi}, \binits{D.}}:
\batitle{A {{Survey}} of {{Methods}} for {{Explaining Black Box Models}}}.
\bjtitle{ACM Computing Surveys (CSUR)}
\bvolume{51}(\bissue{5}),
\bfpage{93}--\blpage{19342}
(\byear{2018})
\doiurl{10.1145/3236009}
\end{barticle}
\endbibitem

\bibitem[\protect\citeauthoryear{Rudin}{2019}]{rudinStopExplainingBlack2019}
\begin{barticle}
\bauthor{\bsnm{Rudin}, \binits{C.}}:
\batitle{Stop explaining black box machine learning models for high stakes
  decisions and use interpretable models instead}.
\bjtitle{Nature Machine Intelligence}
\bvolume{1}(\bissue{5}),
\bfpage{206}--\blpage{215}
(\byear{2019})
\doiurl{10.1038/s42256-019-0048-x}
\end{barticle}
\endbibitem

\bibitem[\protect\citeauthoryear{Makridakis
  et~al.}{2018}]{makridakisM4CompetitionResults2018}
\begin{barticle}
\bauthor{\bsnm{Makridakis}, \binits{S.}},
\bauthor{\bsnm{Spiliotis}, \binits{E.}},
\bauthor{\bsnm{Assimakopoulos}, \binits{V.}}:
\batitle{The {{M4 Competition}}: {{Results}}, findings, conclusion and way
  forward}.
\bjtitle{International Journal of Forecasting}
\bvolume{34}(\bissue{4}),
\bfpage{802}--\blpage{808}
(\byear{2018})
\doiurl{10.1016/j.ijforecast.2018.06.001}
\end{barticle}
\endbibitem

\bibitem[\protect\citeauthoryear{Smyl}{2020}]{smylHybridMethodExponential2020}
\begin{barticle}
\bauthor{\bsnm{Smyl}, \binits{S.}}:
\batitle{A hybrid method of exponential smoothing and recurrent neural networks
  for time series forecasting}.
\bjtitle{International Journal of Forecasting}
\bvolume{36}(\bissue{1}),
\bfpage{75}--\blpage{85}
(\byear{2020})
\doiurl{10.1016/j.ijforecast.2019.03.017}
\end{barticle}
\endbibitem

\bibitem[\protect\citeauthoryear{Cerqueira
  et~al.}{2017}]{cerqueiraDynamicHeterogeneousEnsembles2017}
\begin{bchapter}
\bauthor{\bsnm{Cerqueira}, \binits{V.}},
\bauthor{\bsnm{Torgo}, \binits{L.}},
\bauthor{\bsnm{Oliveira}, \binits{M.}},
\bauthor{\bsnm{Pfahringer}, \binits{B.}}:
\bctitle{Dynamic and {{Heterogeneous Ensembles}} for {{Time Series
  Forecasting}}}.
In: \bbtitle{2017 {{IEEE International Conference}} on {{Data Science}} and
  {{Advanced Analytics}} ({{DSAA}})},
pp. \bfpage{242}--\blpage{251}
(\byear{2017})
\end{bchapter}
\endbibitem

\bibitem[\protect\citeauthoryear{Saadallah}{2023}]{saadallahOnlineExplainableModel2023}
\begin{bchapter}
\bauthor{\bsnm{Saadallah}, \binits{A.}}:
\bctitle{Online {{Explainable Model Selection}} for {{Time Series
  Forecasting}}}.
In: \bbtitle{2023 {{IEEE}} 10th {{International Conference}} on {{Data
  Science}} and {{Advanced Analytics}} ({{DSAA}})},
pp. \bfpage{1}--\blpage{10}
(\byear{2023})
\end{bchapter}
\endbibitem

\bibitem[\protect\citeauthoryear{Priebe}{2019}]{priebeDynamicModelSelection2019}
\begin{botherref}
\oauthor{\bsnm{Priebe}, \binits{F.}}:
Dynamic {{Model Selection}} for {{Automated Machine Learning}} in {{Time
  Series}}.
PhD thesis,
TU Dortmund University
(2019)
\end{botherref}
\endbibitem

\bibitem[\protect\citeauthoryear{Selvaraju
  et~al.}{2019}]{selvarajuGradCAMVisualExplanations2019}
\begin{barticle}
\bauthor{\bsnm{Selvaraju}, \binits{R.R.}},
\bauthor{\bsnm{Cogswell}, \binits{M.}},
\bauthor{\bsnm{Das}, \binits{A.}},
\bauthor{\bsnm{Vedantam}, \binits{R.}},
\bauthor{\bsnm{Parikh}, \binits{D.}},
\bauthor{\bsnm{Batra}, \binits{D.}}:
\batitle{Grad-{{CAM}}: {{Visual Explanations}} from {{Deep Networks}} via
  {{Gradient-Based Localization}}}.
\bjtitle{International Journal of Computer Vision}
\bvolume{128}(\bissue{2}),
\bfpage{336}--\blpage{359}
(\byear{2019})
\doiurl{10.1007/s11263-019-01228-7}
\end{barticle}
\endbibitem

\bibitem[\protect\citeauthoryear{Lundberg
  et~al.}{2020}]{lundbergLocalExplanationsGlobal2020}
\begin{barticle}
\bauthor{\bsnm{Lundberg}, \binits{S.M.}},
\bauthor{\bsnm{Erion}, \binits{G.}},
\bauthor{\bsnm{Chen}, \binits{H.}},
\bauthor{\bsnm{DeGrave}, \binits{A.}},
\bauthor{\bsnm{Prutkin}, \binits{J.M.}},
\bauthor{\bsnm{Nair}, \binits{B.}},
\bauthor{\bsnm{Katz}, \binits{R.}},
\bauthor{\bsnm{Himmelfarb}, \binits{J.}},
\bauthor{\bsnm{Bansal}, \binits{N.}},
\bauthor{\bsnm{Lee}, \binits{S.-I.}}:
\batitle{From local explanations to global understanding with explainable
  {{AI}} for trees}.
\bjtitle{Nature Machine Intelligence}
\bvolume{2}(\bissue{1}),
\bfpage{56}--\blpage{67}
(\byear{2020})
\doiurl{10.1038/s42256-019-0138-9}
\end{barticle}
\endbibitem

\bibitem[\protect\citeauthoryear{Zhang
  et~al.}{2020}]{zhangForecastingAgriculturalCommodity2020}
\begin{barticle}
\bauthor{\bsnm{Zhang}, \binits{D.}},
\bauthor{\bsnm{Chen}, \binits{S.}},
\bauthor{\bsnm{Liwen}, \binits{L.}},
\bauthor{\bsnm{Xia}, \binits{Q.}}:
\batitle{Forecasting {{Agricultural Commodity Prices Using Model Selection
  Framework With Time Series Features}} and {{Forecast Horizons}}}.
\bjtitle{IEEE Access}
\bvolume{8},
\bfpage{28197}--\blpage{28209}
(\byear{2020})
\doiurl{10.1109/ACCESS.2020.2971591}
\end{barticle}
\endbibitem

\bibitem[\protect\citeauthoryear{Prud{\^e}ncio and
  Ludermir}{2004}]{prudencioMetalearningApproachesSelecting2004}
\begin{barticle}
\bauthor{\bsnm{Prud{\^e}ncio}, \binits{R.B.C.}},
\bauthor{\bsnm{Ludermir}, \binits{T.B.}}:
\batitle{Meta-learning approaches to selecting time series models}.
\bjtitle{Neurocomputing}
\bvolume{61},
\bfpage{121}--\blpage{137}
(\byear{2004})
\doiurl{10.1016/j.neucom.2004.03.008}
\end{barticle}
\endbibitem

\bibitem[\protect\citeauthoryear{Talagala
  et~al.}{2023}]{talagalaMetalearningHowForecast2023}
\begin{barticle}
\bauthor{\bsnm{Talagala}, \binits{T.S.}},
\bauthor{\bsnm{Hyndman}, \binits{R.J.}},
\bauthor{\bsnm{Athanasopoulos}, \binits{G.}}:
\batitle{Meta-learning how to forecast time series}.
\bjtitle{Journal of Forecasting}
\bvolume{42}(\bissue{6}),
\bfpage{1476}--\blpage{1501}
(\byear{2023})
\doiurl{10.1002/for.2963}
\end{barticle}
\endbibitem

\bibitem[\protect\citeauthoryear{Hajirahimi and
  Khashei}{2019}]{hajirahimiHybridStructuresTime2019}
\begin{barticle}
\bauthor{\bsnm{Hajirahimi}, \binits{Z.}},
\bauthor{\bsnm{Khashei}, \binits{M.}}:
\batitle{Hybrid structures in time series modeling and forecasting: {{A}}
  review}.
\bjtitle{Engineering Applications of Artificial Intelligence}
\bvolume{86},
\bfpage{83}--\blpage{106}
(\byear{2019})
\doiurl{10.1016/j.engappai.2019.08.018}
\end{barticle}
\endbibitem

\bibitem[\protect\citeauthoryear{Zhang}{2003}]{zhangTimeSeriesForecasting2003}
\begin{barticle}
\bauthor{\bsnm{Zhang}, \binits{G.P.}}:
\batitle{Time series forecasting using a hybrid {{ARIMA}} and neural network
  model}.
\bjtitle{Neurocomputing}
\bvolume{50},
\bfpage{159}--\blpage{175}
(\byear{2003})
\doiurl{10.1016/S0925-2312(01)00702-0}
\end{barticle}
\endbibitem

\bibitem[\protect\citeauthoryear{Brehler and
  Camphausen}{2023}]{brehler2023combiningDecisionTreeCNN}
\begin{bchapter}
\bauthor{\bsnm{Brehler}, \binits{M.}},
\bauthor{\bsnm{Camphausen}, \binits{L.}}:
\bctitle{Combining decision tree and convolutional neural network for energy
  efficient on-device activity recognition}.
In: \bbtitle{2023 IEEE 16th International Symposium on Embedded
  Multicore/Many-core Systems-on-Chip (MCSoC)},
pp. \bfpage{179}--\blpage{185}.
\bpublisher{IEEE Computer Society},
\blocation{Los Alamitos, CA, USA}
(\byear{2023}).
\doiurl{10.1109/MCSoC60832.2023.00035} .
\burl{https://doi.ieeecomputersociety.org/10.1109/MCSoC60832.2023.00035}
\end{bchapter}
\endbibitem

\bibitem[\protect\citeauthoryear{Daghero
  et~al.}{2022}]{dagheroTwostageHumanActivity2022}
\begin{bchapter}
\bauthor{\bsnm{Daghero}, \binits{F.}},
\bauthor{\bsnm{Pagliari}, \binits{D.J.}},
\bauthor{\bsnm{Poncino}, \binits{M.}}:
\bctitle{Two-stage {{Human Activity Recognition}} on {{Microcontrollers}} with
  {{Decision Trees}} and {{CNNs}}}.
In: \bbtitle{2022 17th {{Conference}} on {{Ph}}.{{D Research}} in
  {{Microelectronics}} and {{Electronics}} ({{PRIME}})},
pp. \bfpage{173}--\blpage{176}
(\byear{2022}).
\doiurl{10.1109/PRIME55000.2022.9816745}
\end{bchapter}
\endbibitem

\bibitem[\protect\citeauthoryear{Buschj{\"a}ger}{2024}]{buschjagerRejectionEnsemblesOnline2024}
\begin{bchapter}
\bauthor{\bsnm{Buschj{\"a}ger}, \binits{S.}}:
\bctitle{Rejection {{Ensembles}} with~{{Online Calibration}}}.
In: \beditor{\bsnm{Bifet}, \binits{A.}},
\beditor{\bsnm{Davis}, \binits{J.}},
\beditor{\bsnm{Krilavi{\v c}ius}, \binits{T.}},
\beditor{\bsnm{Kull}, \binits{M.}},
\beditor{\bsnm{Ntoutsi}, \binits{E.}},
\beditor{\bsnm{{\v Z}liobait{\.e}}, \binits{I.}} (eds.)
\bbtitle{Machine {{Learning}} and {{Knowledge Discovery}} in {{Databases}}.
  {{Research Track}}},
pp. \bfpage{3}--\blpage{20}.
\bpublisher{Springer},
\blocation{Cham}
(\byear{2024}).
\doiurl{10.1007/978-3-031-70365-2_1}
\end{bchapter}
\endbibitem

\bibitem[\protect\citeauthoryear{Hyndman and
  Athanasopoulos}{2021}]{hyndmanForecastingPrinciplesPractices2021}
\begin{bbook}
\bauthor{\bsnm{Hyndman}, \binits{R.J.}},
\bauthor{\bsnm{Athanasopoulos}, \binits{G.}}:
\bbtitle{Forecasting: {{Principles}} and {{Practices}}},
\bedition{3}rd edn.
\bpublisher{OTexts: Melbourne, Australia}, \blocation{???}
(\byear{2021})
\end{bbook}
\endbibitem

\bibitem[\protect\citeauthoryear{Salinas
  et~al.}{2019}]{salinasDeepARProbabilisticForecasting2019}
\begin{botherref}
\oauthor{\bsnm{Salinas}, \binits{D.}},
\oauthor{\bsnm{Flunkert}, \binits{V.}},
\oauthor{\bsnm{Gasthaus}, \binits{J.}}:
{{DeepAR}}: {{Probabilistic Forecasting}} with {{Autoregressive Recurrent
  Networks}}.
arXiv
(2019).
\doiurl{10.48550/arXiv.1704.04110}
\end{botherref}
\endbibitem

\bibitem[\protect\citeauthoryear{Dem{\v{s}}ar}{2006}]{demsar2006statistical}
\begin{barticle}
\bauthor{\bsnm{Dem{\v{s}}ar}, \binits{J.}}:
\batitle{Statistical comparisons of classifiers over multiple data sets}.
\bjtitle{The Journal of Machine learning research}
\bvolume{7}(\bissue{1}),
\bfpage{1}--\blpage{30}
(\byear{2006})
\end{barticle}
\endbibitem

\bibitem[\protect\citeauthoryear{Benavoli et~al.}{2016}]{benavoli2016should}
\begin{barticle}
\bauthor{\bsnm{Benavoli}, \binits{A.}},
\bauthor{\bsnm{Corani}, \binits{G.}},
\bauthor{\bsnm{Mangili}, \binits{F.}}:
\batitle{Should we really use post-hoc tests based on mean-ranks?}
\bjtitle{The Journal of Machine Learning Research}
\bvolume{17}(\bissue{1}),
\bfpage{152}--\blpage{161}
(\byear{2016})
\end{barticle}
\endbibitem

\bibitem[\protect\citeauthoryear{Forman and
  Scholz}{2010}]{formanApplesapplesCrossvalidationStudies2010}
\begin{barticle}
\bauthor{\bsnm{Forman}, \binits{G.}},
\bauthor{\bsnm{Scholz}, \binits{M.}}:
\batitle{Apples-to-apples in cross-validation studies: Pitfalls in classifier
  performance measurement}.
\bjtitle{SIGKDD Explor. Newsl.}
\bvolume{12}(\bissue{1}),
\bfpage{49}--\blpage{57}
(\byear{2010})
\doiurl{10.1145/1882471.1882479}
\end{barticle}
\endbibitem

\end{thebibliography}

\end{document}